\documentclass{article}

\usepackage{PRIMEarxiv}

\usepackage[T1]{fontenc}
\usepackage[utf8]{inputenc}

\usepackage{amsmath,amssymb,amsfonts,mathtools,amsthm}
\usepackage{array}
\usepackage{booktabs}
\usepackage{enumitem}
\usepackage{graphicx}
\usepackage{microtype}
\usepackage{multirow}
\usepackage{tabularx}
\usepackage{makecell}
\usepackage{pifont}
\usepackage{xcolor}
\usepackage{url}
\usepackage{natbib}
\usepackage{hyperref}

\definecolor{darkblue}{rgb}{0, 0, 0.5}
\definecolor{paperteal}{RGB}{48,121,137}
\definecolor{paperteallight}{RGB}{235,244,246}
\definecolor{papergraylight}{RGB}{244,244,244}

\hypersetup{
  colorlinks=true,
  linkcolor=darkblue,
  urlcolor=darkblue,
  citecolor=darkblue,
  pdftitle={Contextual Multi-Objective Optimization: Rethinking Objectives in Frontier AI Systems}
}

\newtheorem{definition}{Definition}

\newtheorem{proposition}{Proposition}

\title{Contextual Multi-Objective Optimization: Rethinking Objectives in Frontier AI Systems}

\author{
  Jie Zhou$^{1}$, Qin Chen$^1$, Liang He$^1$ \\
  $^1$ School of Computer Science and Technology, East China Normal University, Shanghai\\
  \texttt{\{jzhou, qchen, lhe\}@cs.ecnu.edu.cn}
}

\begin{document}
\maketitle

\begin{abstract}
Frontier AI systems perform best in settings with clear, stable, and verifiable objectives, such as code generation, mathematical reasoning, games, and unit-test-driven tasks. They remain less reliable in open-ended settings, including scientific assistance, long-horizon agents, high-stakes advice, personalization, and tool use, where the relevant objective is ambiguous, context-dependent, delayed, or only partially observable. We argue that many such failures are not merely failures of scale or capability, but failures of objective selection: the system optimizes a locally visible signal while missing which objectives should govern the interaction.

We formulate this problem as \emph{contextual multi-objective optimization}. In this setting, systems must consider multiple, context-dependent objectives, such as helpfulness, truthfulness, safety, privacy, calibration, non-manipulation, user preference, reversibility, and stakeholder impact, while determining which objectives are active, which are soft preferences, and which must function as hard or quasi-hard constraints. These examples are not intended as an exhaustive taxonomy: different domains and deployment settings may activate different objective dimensions and different conflict-resolution procedures. We distinguish approximate objective reducibility, where scalar optimization is empirically justified, from objective irreducibility, where scalar rewards fail to preserve context-dependent constraints, stakeholder scope, or legitimate trade-offs.

Our framework models AI behavior as a context-dependent choice rule over candidate actions, objective estimates, active constraints, stakeholders, uncertainty, and conflict-resolution procedures. This view treats clarification, refusal, uncertainty disclosure, confirmation, escalation, tool use, and abstention as endogenous action classes rather than interface afterthoughts. We outline an implementation pathway based on decomposed objective representations, context-to-objective routing, hierarchical constraints, deliberative policy reasoning, controlled personalization, tool-use control, diagnostic evaluation, auditing, and post-deployment revision.
\end{abstract}

\section{Introduction}
\label{sec:introduction}

Modern AI systems have made remarkable progress on tasks with clear, stable, and verifiable objectives. Large language models benefit from next-token prediction at scale; game-playing agents optimize explicit win conditions; code models can be evaluated by executable tests; and mathematical reasoning systems can often be checked by final-answer correctness or formal verification. In such domains, the optimization target is relatively well specified, feedback is frequent or scalable, and success can be directly measured. This paradigm has enabled modern systems to achieve strong performance and, in some cases, match or exceed human-level capability in code generation, mathematical reasoning, and game playing \citep{vaswani2017attention,brown2020language,kaplan2020scaling,hoffmann2022training}.

However, the same paradigm remains much less reliable in open-ended settings where the objective is ambiguous, context-dependent, delayed, or only partially observable. Frontier AI systems are increasingly expected to support scientific research, operate long-horizon agents, provide high-stakes advice, personalize assistance, use external tools, and interact with multiple stakeholders. In these settings, the central difficulty is often not generating a plausible response, but determining what kind of behavior is appropriate. A system may produce an answer that is fluent but unsafe, preferred but false, personalized but privacy-violating, efficient but poorly calibrated, or locally useful while imposing risks on absent third parties. Many deployed failures therefore arise not because the model lacks the ability to optimize, but because it optimizes the wrong objective for the situation.

We argue that this gap reflects a deeper problem: many frontier AI tasks are instances of \emph{contextual multi-objective optimization}. In such tasks, systems must jointly consider helpfulness, truthfulness, safety, privacy, calibration, non-manipulation, user preference, reversibility, and the interests of affected stakeholders. Which objectives are active, how they should be prioritized, and whether they may be traded off depend on the context. Existing alignment and optimization methods often compress these requirements into a scalar reward, a preference score, or an averaged human judgment. This reduction can work when objectives are measurable, stable, commensurable, and legitimately tradeable. But it fails when objectives conflict, when constraints activate only in certain contexts, or when some requirements must be treated as non-tradeable constraints. A private-data violation should not become acceptable because the response is helpful; a false answer should not be reinforced because it is reassuring; and harmful operational guidance should not be supplied because it satisfies the immediate user.

We use the term \emph{contextual multi-objective optimization} not to denote only vector-reward optimization, but to denote a broader choice-rule problem in which the system must identify the applicable objective structure before optimizing within it. The novelty therefore lies not in replacing scalar rewards with vector rewards alone, but in modeling objective activation, constraint status, stakeholder scope, uncertainty, and action-class selection as part of the decision problem.

This paper studies contextual multi-objective optimization as a formal problem for frontier AI systems. We distinguish tasks that are approximately reducible to scalar optimization from tasks that are \emph{objective-irreducible}, where a single scalar metric cannot faithfully represent the relevant constraints, stakeholders, uncertainty, or legitimate trade-offs. We formulate the problem as a choice rule over contexts, candidate actions, objective estimates, active constraints, stakeholders, uncertainty, and conflict-resolution procedures. This view allows scalar optimization as a special case, but makes explicit the assumptions under which it is valid.

Our formulation also clarifies why stronger reward models alone are insufficient. A frontier AI system must not only estimate how useful, safe, truthful, or preferred an action is; it must also determine which objectives matter in the current context, which constraints override preferences, and what class of action is appropriate. In many cases, the correct behavior is not to answer directly, but to ask for clarification, refuse, disclose uncertainty, request confirmation, preserve user agency, or escalate. These behaviors should be treated as policy actions induced by the objective structure, rather than as conversational imperfections.

The paper makes the following contributions.
\begin{itemize}[leftmargin=1.4em,itemsep=0.25em]
    \item We identify \emph{objective selection failure} as a distinct failure mode of frontier AI systems, separate from capability failure and proxy-estimation error.
    \item We formalize contextual multi-objective optimization as a choice-rule problem in which active objectives, constraints, stakeholders, uncertainty, and action classes depend on context.
    \item We introduce approximate objective reducibility and empirical objective irreducibility to clarify when scalar optimization is justified and when it breaks down.
    \item We propose an implementation pathway that connects decomposed objective modeling, context-to-objective routing, hierarchical constraints, deliberative reasoning, tool-use control, personalization control, diagnostic evaluation, and post-deployment revision.
\end{itemize}

\begin{table}[t]
\centering
\small
\begin{tabularx}{\linewidth}{p{0.10\linewidth}p{0.15\linewidth}p{0.25\linewidth}X}
\toprule
Existing approach & Main optimization signal & Why it works & Why it is insufficient for frontier systems \\
\midrule
Pretraining & Next-token prediction over large-scale data & Learns broad linguistic, factual, and reasoning patterns from scalable data. & The objective is predictive rather than normative: it teaches what text is likely, not which behavior is appropriate, safe, or contextually justified. \\
\addlinespace
Supervised fine-tuning & Imitation of curated demonstrations & Improves instruction following and encodes desirable response styles through human-written examples. & It depends on the coverage of demonstrations and cannot reliably specify how to resolve novel conflicts among truthfulness, safety, privacy, and user preference. \\
\addlinespace
RLHF and preference optimization & Scalar reward or preference score from human judgments & Aligns outputs with perceived helpfulness and user satisfaction, often improving interaction quality. & It compresses heterogeneous objectives into a single score and may reward answers that are preferred, fluent, or reassuring but false, unsafe, or poorly calibrated. \\
\addlinespace
Verifier- or outcome-based RL & Task-specific correctness signals, tests, or formal checks & Highly effective when success is clear and directly verifiable, such as code, mathematics, or games. & Many real-world tasks have delayed, ambiguous, or partially observable consequences, so the relevant objective cannot be reduced to immediate correctness or task completion. \\
\addlinespace
Agentic optimization & Long-horizon task success, efficiency, or tool-use outcomes & Enables systems to plan, use tools, and complete complex multi-step tasks. & Greater autonomy amplifies objective-selection failures: systems may optimize efficiency while ignoring reversibility, consent, third-party impact, or escalation requirements. \\
\bottomrule
\end{tabularx}
\caption{Limitations of existing optimization approaches for contextual multi-objective optimization in open-ended frontier AI systems. Representative approaches include large-scale pretraining and scaling \citep{vaswani2017attention,brown2020language,kaplan2020scaling,hoffmann2022training}, RLHF and preference optimization \citep{christiano2017deep,ouyang2022training,rafailov2023direct,azar2024general}, constitutional and AI-feedback methods \citep{bai2022constitutional}, and scalable oversight or verifier-style approaches for difficult-to-evaluate tasks \citep{irving2018ai,burns2023weak}.}
\label{tab:existing_limitations}
\end{table}

\section{Background}
\label{sec:background}

\subsection{Optimization Objectives in Current Frontier AI}

The progress of frontier AI systems has been closely tied to the design of optimization objectives. Pretraining optimizes next-token prediction over large-scale corpora, a paradigm enabled by the Transformer architecture and later scaled to increasingly capable language models \citep{vaswani2017attention,brown2020language,kaplan2020scaling,hoffmann2022training}. Supervised fine-tuning (SFT) teaches models to imitate curated demonstrations, while reinforcement-learning-based methods, including RLHF, RLAIF, constitutional training, verifier-based reinforcement learning, and preference optimization, further shape model behavior according to human judgments, AI feedback, explicit rules, or task-specific outcome signals \citep{christiano2017deep,ouyang2022training,bai2022constitutional,rafailov2023direct,azar2024general}. These methods have been highly successful in settings where success is relatively clear, stable, and verifiable, such as code generation, mathematical reasoning, game playing, and unit-test-driven problem solving.

However, these methods also share a structural assumption: desirable behavior can be represented, either explicitly or implicitly, by a relatively fixed optimization signal. This signal may be a demonstration distribution, a scalar reward, a preference score, a rule-based evaluator, or a task-specific verifier. The assumption is often reasonable when the task has a stable objective and immediate feedback. It becomes much weaker when the system must act in open-ended, high-stakes, personalized, or agentic settings, where the relevant objectives are ambiguous, context-dependent, delayed, or only partially observable. Evaluation work such as HELM and TruthfulQA has similarly shown that model quality cannot be adequately captured by a single aggregate score, because robustness, calibration, toxicity, truthfulness, fairness, and efficiency can vary independently \citep{liang2022holistic,lin2021truthfulqa}. Table~\ref{tab:existing_limitations} summarizes the main advantages and limitations of existing optimization approaches.

\subsection{From Scalar Optimization to Contextual Multi-Objective Optimization}

The limitation of existing methods is not that scalar objectives are always wrong. Rather, scalar optimization is valid only under specific assumptions: the relevant objectives must be measurable, commensurable, stable across contexts, and legitimately tradeable. Classical multi-objective optimization and multi-objective reinforcement learning provide formal tools for vector rewards, Pareto fronts, scalarization, and preference-conditioned policies \citep{miettinen1999nonlinear,roijers2013survey,hayes2022practical,qiu2024traversing}. These tools are valuable once the objective dimensions and decision procedure are specified. The difficulty in frontier AI deployment is often prior to that stage: determining what the relevant objectives are, who authorizes them, which objectives are soft preferences, and which must operate as hard or quasi-hard constraints.

In well-defined tasks, scalar or outcome-based optimization may be a good approximation. A code model can be optimized against executable tests, a game-playing agent can optimize a win condition, and a mathematical reasoning model can often be evaluated by final-answer correctness or formal verification. In such settings, the objective is relatively clear, feedback is available, and optimization pressure is directed toward a stable target. Open-ended frontier AI systems violate these assumptions in systematic ways. A single user request may instantiate different objective profiles depending on context. The same question can be a benign information request, a privacy-sensitive request, a medical-style advice request, or a high-risk tool-use request. In these settings, the central challenge is not merely to produce a plausible answer, but to determine what kind of objective problem the system is facing.

A response may be helpful but false, preferred but unsafe, personalized but privacy-violating, efficient but irreversible, or satisfying to the immediate user while harmful to absent stakeholders. Preference learning and RLHF recognize that desired behavior is difficult to fully specify by hand \citep{ng2000algorithms,christiano2017deep,ouyang2022training}. Yet in many implementations, heterogeneous reasons for preference are compressed into a single reward or comparison signal. A response may be preferred because it is more truthful, more cautious, more polite, shorter, or more aligned with annotator expectations. If that reason structure is not preserved, the system may learn average preference satisfaction without learning how to resolve conflicts among truthfulness, safety, privacy, and stakeholder impact. Thus, improving a scalar reward model or scaling model capability does not by itself solve the deeper problem of objective selection.

Current AI optimization lacks an explicit mechanism for objective-structure identification. Pretraining, SFT, RLHF, preference optimization, and verifier-based training can improve capability or align behavior with observed preferences, but they typically assume that the relevant objective structure has already been encoded in the training signal, reward model, rule set, or evaluator. Constitutional AI and deliberative alignment move toward more explicit principles and safety specifications, but principles still require contextual interpretation, conflict resolution, and revision when failures appear \citep{bai2022constitutional,guan2024deliberative}. As a result, the system may learn how to optimize a given signal without learning when that signal is incomplete, when a different objective should dominate, or when a hard constraint should override user preference. This gap becomes increasingly consequential as frontier systems become more capable, personalized, and agentic: a system may become more helpful while becoming more privacy-invasive, more efficient while becoming less corrigible, or more persuasive while undermining user autonomy.

We therefore frame open-ended frontier AI behavior as \emph{contextual multi-objective optimization}. In this view, an AI system must select an action by considering multiple objectives whose relevance, priority, and constraint status depend on the interaction context. The action space is not limited to direct answers. It may also include asking for clarification, refusing, disclosing uncertainty, requesting confirmation, escalating to human review, using a tool, or taking no action. These behaviors should not be treated as peripheral interface choices; they are part of the optimization problem because they determine how the system acts when objectives conflict or when direct answering is inappropriate.

Informally, let $c$ denote an interaction context, $\mathcal{A}(c)$ the set of candidate actions, and $\mathbf{L}(a,c)$ a vector of objective-relevant estimates for action $a$, such as helpfulness, truthfulness, safety risk, privacy risk, uncertainty, user preference, reversibility, calibration, or third-party impact. A vector of scores alone does not specify what should be done. A system also needs a context-dependent choice rule that determines which objectives are active, which constraints override preferences, whose interests count, and how conflicts should be resolved. This is the core departure from ordinary scalar optimization: the system must not only optimize over actions, but also identify the objective structure that governs the current context.

\subsection{Positioning Relative to Existing Frameworks}

Contextual multi-objective optimization builds on, but differs from, several existing frameworks. Unlike standard multi-objective optimization and multi-objective reinforcement learning, it does not assume that the objective dimensions and scalarization or selection procedure are already specified. Unlike constrained reinforcement learning or safe reinforcement learning, it treats constraint activation itself as context-dependent, uncertain, and partly normative. Unlike ordinary preference learning, it preserves reasons for preference and distinguishes immediate user preference from constraints associated with truthfulness, privacy, safety, legality, autonomy, or third-party protection. Unlike policy-only safety layers, it treats clarification, refusal, uncertainty disclosure, confirmation, escalation, and abstention as endogenous actions in the decision problem rather than as external filters applied after generation.

The point is not that scalar rewards, vector rewards, constraints, preferences, or rules are obsolete. Each remains useful under the right assumptions. The central claim is that open-ended frontier AI requires an additional layer of objective-structure identification: before optimizing, the system must determine which objectives matter, which objectives are soft or hard in context, whose interests count, what uncertainty remains, and what action classes are admissible.

\subsection{Why Contextual Multi-Objective Optimization Is Hard}

Contextual multi-objective optimization is difficult because frontier AI objectives are not merely numerous; they differ in type, authority, observability, and temporal structure. First, many objectives are \emph{open-textured}. Terms such as helpful, safe, honest, fair, respectful, and non-manipulative do not have fixed boundaries across all domains. Helpfulness may require satisfying a request in a creative task, redirecting the user in a dangerous task, or disclosing uncertainty in a high-stakes advice setting. Fairness provides a concrete example: different formal criteria can be mutually incompatible under realistic conditions \citep{hardt2016equality,kleinberg2016inherent}. As a result, disagreement among annotators may reflect not only noise, but also expertise, cultural context, domain norms, or legitimate normative conflict. Collapsing such disagreement into an average label can remove information that the system needs in order to act appropriately.

Second, some objectives are \emph{incommensurable} or \emph{non-tradeable} in specific contexts. A response can be more personalized but less private, more persuasive but more manipulative, or more concise but less transparent. A weighted sum treats these differences as compensable, but many deployment contexts reject that assumption. A privacy violation should not become acceptable because the answer is helpful; a false answer should not be reinforced because the user finds it reassuring; and harmful operational guidance should not be provided simply because it satisfies the immediate request. Such cases require procedural structure, such as consent, review, professional norms, conservative defaults, or rights-like constraints, rather than unconstrained reward maximization. Work on AI safety and human-compatible AI has similarly emphasized that specification, oversight, and constraint design are central to avoiding unintended optimization behavior \citep{amodei2016concrete,russell2019human}.

Third, objectives are often \emph{hierarchical}. Some objectives are performance goals, such as answer quality, efficiency, and style. Others are hard or quasi-hard constraints under specified risk regimes, such as safety, privacy, legality, and non-deception. Still others are meta-objectives, such as calibration, corrigibility, auditability, and preservation of user agency. Treating all of them as weighted loss terms risks a category error: active constraints should not be compensated by improvements in soft preferences. This distinction becomes increasingly important as systems become more autonomous and act across longer time horizons. Multi-objective RL can help learn policies across trade-off surfaces, but it typically presupposes that the reward dimensions and scalarization or selection procedure have already been specified \citep{roijers2013survey,hayes2022practical,qiu2024traversing}.

Fourth, proxy objectives can fail under optimization pressure. Metrics are useful evidence, but they are not final authority. Goodhart's law states that when a measure becomes a target, it can cease to be a good measure \citep{goodhart1975problems}. Specification gaming gives concrete examples in which agents satisfy a literal reward while violating the designer's intent \citep{krakovna2020specification}. In language-model alignment, optimizing against a learned reward model can improve the proxy while degrading true human preference beyond a certain point \citep{gao2023scaling}. This problem is especially severe in frontier AI systems because the model may be capable enough to satisfy the visible metric while violating the intended norm. Therefore, metrics must be stress-tested, monitored for drift, and revised after failures rather than treated as fixed definitions of desirable behavior.

Finally, objectives are distributed across time and stakeholders. Some effects cannot be measured at the moment of interaction: a tutoring system may appear helpful while weakening long-term learning, a persuasive assistant may satisfy a user while reducing autonomy, and a medical-style answer may appear plausible to a non-expert while being dangerously wrong. Scalable oversight research addresses part of this problem by studying how humans can evaluate model behavior that is difficult to check directly \citep{irving2018ai,burns2023weak}. However, oversight still requires criteria for what counts as a good resolution. Moreover, the immediate user is not always the only affected party. A request may involve another person's data, a future user, an institutional obligation, or a public safety concern. Preference learning from the present user is therefore insufficient as a complete objective source, and social choice theory shows that aggregating plural preferences into a single collective ordering is itself a difficult problem \citep{arrow1951social,sen1970collective}.

\begin{table}[t]
\centering
\small
\begin{tabularx}{\linewidth}{p{0.28\linewidth}X}
\toprule
Challenge & Implication for contextual multi-objective optimization \\
\midrule
Open-textured objectives & Labels may reflect legitimate disagreement rather than noise; systems need context-sensitive interpretation rather than simple averaging. \\
\addlinespace
Incommensurable objectives & Some trade-offs are illegitimate; usefulness or preference should not compensate for privacy, safety, legality, or truthfulness failures in contexts where these constraints are active. \\
\addlinespace
Hierarchical objectives & Performance goals, active constraints, and meta-objectives should not be treated as flat weighted terms in a single loss. \\
\addlinespace
Proxy failure under optimization & Reward models and metrics must be treated as fallible evidence, stress-tested under distribution shift, and revised after failures. \\
\addlinespace
Delayed and unverifiable feedback & Immediate user satisfaction may not capture long-term learning, autonomy, safety, or downstream consequences. \\
\addlinespace
Distributed stakeholders & User preference alone may miss third parties, future users, institutional duties, or public safety concerns. \\
\bottomrule
\end{tabularx}
\caption{Why contextual multi-objective optimization is difficult for frontier AI systems. The challenge is not only to estimate multiple objective scores, but to determine which objectives are active, which are constraints, and how conflicts should be resolved in context.}
\label{tab:cmoo_challenges}
\end{table}

\subsection{Toward Objective Mechanics}

The need for contextual multi-objective optimization suggests a broader research perspective: understanding not only how models learn capabilities, but how objectives shape deployed behavior under competing pressures. We refer to this perspective as \emph{objective mechanics}. Objective mechanics studies how objectives are represented, activated, traded off, constrained, and revised in trained systems and deployed products. Its basic unit is not a single prompt or a single metric, but a conflict pattern: helpfulness against safety, personalization against privacy, persuasion against autonomy, speed against calibration, or user preference against third-party protection.

This framing connects capability and alignment. Capability research asks whether a model can perform a task. Alignment research asks whether it performs the intended task in the intended way under relevant constraints. Objective mechanics treats failures such as hallucination, over-refusal, privacy leakage, jailbreak susceptibility, sycophancy, and unsafe tool use as evidence about how the system resolves competing objective pressures. Red-teaming, sycophancy analysis, and jailbreak research provide empirical evidence that such failures often emerge when models face competing pressures between helpfulness, user agreement, policy boundaries, and robustness \citep{perez2022red,sharma2023sycophancy,zou2023universal}. Evaluation, interpretability, and documentation are therefore complementary tools: evaluation maps failures, interpretability probes mechanisms, and documentation records intended use and accountability assumptions \citep{liang2022holistic,lin2021truthfulqa,olah2020circuits,elhage2022toy,mitchell2019model,raji2020closing}.

The importance of this paper lies in making this objective structure explicit. Existing methods ask how AI systems can better optimize a given signal. We argue that frontier AI systems increasingly require a prior question: \emph{what should the system optimize in this context?} By formulating open-ended AI behavior as contextual multi-objective optimization, this work clarifies when scalar optimization remains a useful approximation, why it breaks down in objective-irreducible settings, and what kinds of choice rules are needed for systems that can identify active objectives, enforce non-tradeable constraints, account for affected stakeholders, disclose uncertainty, escalate when necessary, and revise objective procedures after deployment failures.

\section{Problem Formulation of Contextual Multi-Objective Optimization}
\label{sec:problem}

We formulate open-ended frontier AI behavior as a problem of \emph{contextual multi-objective optimization}. The core question is not only how to optimize a given objective, but how a system should determine which objectives are relevant, how they should be prioritized, and which of them must function as hard or quasi-hard constraints in a particular interaction context.

Let $c \in \mathcal{X}$ denote an interaction context. The context may include the user request, conversation history, domain, risk level, available tools, user state, institutional setting, and relevant external conditions. Let $\mathcal{A}(c)$ denote the set of candidate actions available to the system in context $c$. Importantly, candidate actions are not limited to direct answers. They may also include asking for clarification, refusing, disclosing uncertainty, requesting confirmation, escalating to human review, using a tool, or taking no action.

Let $\bar{\mathcal{O}}$ denote a broad universe of possible objective dimensions, and let $\mathcal{O}_{\mathrm{active}}(c)\subseteq \bar{\mathcal{O}}$ denote the set of objectives activated in context $c$. For each candidate action $a \in \mathcal{A}(c)$, the system estimates objective-relevant quantities:
\begin{equation}
    \mathbf{L}_{\mathcal{O}_{\mathrm{active}}(c)}(a,c)
    =
    \big(L_i(a,c)\big)_{i\in \mathcal{O}_{\mathrm{active}}(c)} ,
\end{equation}
where each component may correspond to helpfulness, truthfulness, safety risk, privacy risk, uncertainty, user preference, reversibility, calibration, or expected impact on third parties. These quantities describe how an action relates to multiple objectives, but they do not by themselves determine what the system should do. A decision rule must also specify which objectives are active in context, which objectives are tradeable, which constraints are non-tradeable or quasi-hard, whose interests count, and how conflicts should be resolved.

We therefore define the system's behavior through a context-dependent choice rule:
\begin{equation}
    a^*(c) \in
    G\Big(
        \mathcal{A}(c),
        \mathbf{L},
        \mathcal{C}(c),
        \mathcal{S}(c),
        \mathcal{U}(c),
        \Pi
    \Big),
\end{equation}
where $\mathcal{C}(c)$ is the set of active constraints, $\mathcal{S}(c)$ is the set of relevant stakeholders or legitimate preference sources, $\mathcal{U}(c)$ represents uncertainty about the context, objectives, or consequences, and $\Pi$ is the conflict-resolution procedure that governs how objectives and constraints are applied.

This formulation generalizes scalar optimization rather than rejecting it. In simple settings, $G$ may reduce to maximizing a scalar reward or minimizing a weighted sum of losses. In open-ended settings, however, $G$ may involve constraint filtering, abstention, escalation, user confirmation, uncertainty disclosure, or other policy actions. These behaviors are not peripheral interface details; they are part of the optimization problem because they determine how the system acts when objectives conflict or when direct answering is inappropriate.

\paragraph{When Is Scalar Optimization Sufficient?}

Scalar optimization is appropriate only when a single metric faithfully approximates the relevant context-dependent decision judgment. We make this condition explicit through approximate objective reducibility.

\begin{definition}[Approximate objective reducibility]
Let $D$ be a deployment distribution over contexts, $\mathcal{T}$ a set of stress distributions, $h_c(a)$ a human or institutional judgment over candidate actions in context $c$, and $m(a,c)$ a scalar metric or reward. A task is \emph{$(\rho,\epsilon,\delta_r,\delta_c)$-reducible} to $m$ over action class $\mathcal{A}$ if the following conditions hold:
\begin{itemize}[leftmargin=1.5em,itemsep=0.2em]
    \item \textbf{Ranking preservation:} on $D$, the context-wise ranking of candidate actions induced by $m(\cdot,c)$ agrees with the ranking induced by $h_c(\cdot)$ with average rank correlation at least $\rho$.
    \item \textbf{Constraint reliability:} the policy $\pi_m(c)\in\arg\max_{a\in\mathcal{A}(c)}m(a,c)$ violates active constraints with probability at most $\epsilon$ on $D$.
    \item \textbf{Robustness under stress:} on every stress distribution $T\in\mathcal{T}$, ranking correlation decreases by at most $\delta_r$ and constraint-violation probability increases by at most $\delta_c$.
\end{itemize}
\end{definition}

Approximate reducibility treats scalar optimization as an empirical hypothesis rather than a default assumption. A reward model, preference score, or benchmark metric may be useful on ordinary examples, while failing under distribution shift, adversarial prompting, long-horizon consequences, or changes in stakeholder context. Thus, a scalar objective is justified only when it preserves the relevant action ordering, controls constraint violations, and remains robust under stress.

\paragraph{When Does Scalar Optimization Break Down?}

Many frontier AI tasks are not well described by approximate reducibility. They involve multiple objectives that are context-dependent, partially observable, conflicting, or non-tradeable. In such cases, the failure of a scalar reward is not merely a measurement error; the deeper issue is that the task requires contextual activation and prioritization of objectives.

Objective irreducibility has both a structural and an empirical aspect. Structurally, some objectives are not legitimately compensable by gains in other objectives: for example, a privacy violation should not become acceptable merely because an answer is helpful. Empirically, even when a scalar proxy is allowed, it may fail to preserve judgments and active constraints under deployment and stress distributions. The following definition captures the empirical aspect.

\begin{definition}[Empirical objective irreducibility]
Given a class of scalar metrics $\mathcal{M}$, a task exhibits \emph{empirical objective irreducibility} over $(D,\mathcal{T})$ if no $m \in \mathcal{M}$ satisfies approximate reducibility at acceptable thresholds $(\rho,\epsilon,\delta_r,\delta_c)$, or if context-dependent non-tradeable constraints require different action classes for cases that appear equivalent under the scalar metric.
\end{definition}

Objective irreducibility captures the situations in which improving the reward model is not enough. For example, two actions may receive similar helpfulness or preference scores, but one may violate privacy, create irreversible external effects, mislead the user, or harm an absent third party. If such constraints are active only in certain contexts, then a context-blind scalar score cannot determine the correct action class.

\paragraph{Operational Judgment as the Target Behavior}

The behavioral capability required for contextual multi-objective optimization is what we call \emph{operational judgment}.

\begin{definition}[Operational judgment]
A system exhibits \emph{operational judgment} in context $c$ if it can identify the active objectives, recognize conflicts among them, apply non-tradeable or quasi-hard constraints, calibrate uncertainty, select an appropriate action class, and provide a concise rationale or escalation path when needed.
\end{definition}

This is a behavioral definition, not a claim that the system possesses human-like moral understanding. A system with operational judgment need not have human wisdom; it must reliably act as if the relevant objectives and constraints have been correctly identified, prioritized, and applied. In this sense, operational judgment is the practical target of contextual multi-objective optimization.

\paragraph{Why Context-Blind Scalarization Fails}

The preceding definitions clarify why weighted scalarization is not a neutral default. It assumes that objectives are commensurable, that trade-offs among them are legitimate, that weights remain stable across contexts, and that proxy measurements remain valid under optimization pressure. The following proposition gives a minimal failure mode.

\begin{proposition}[A minimal failure mode for context-blind scalarization]
Consider two contexts $c_1,c_2$ and two candidate actions $a,b$. Suppose a scalarization rule
\begin{equation}
    s_w(a,c)=w^\top \mathbf{L}(a,c)
\end{equation}
observes only the objective-score vector $\mathbf{L}(a,c)$. Assume that
\begin{equation}
    \mathbf{L}(a,c_1)=\mathbf{L}(a,c_2),
    \qquad
    \mathbf{L}(b,c_1)=\mathbf{L}(b,c_2).
\end{equation}
If the legitimate context-dependent procedure requires $a \succ b$ in $c_1$ but $b \succ a$ in $c_2$ because a hard or quasi-hard constraint is active only in one context, then no context-blind weighted scalarization over $\mathbf{L}$ can satisfy both requirements.
\end{proposition}

\begin{proof}
Because the scalarization rule observes identical score vectors for $a$ and $b$ in both contexts, it must assign the same scalar scores and therefore induce the same ordering between $a$ and $b$ in $c_1$ and $c_2$. The context-dependent requirements demand opposite orderings. Thus at least one context must be misclassified. The failure is not caused by an incorrect choice of weights; it arises because the scalarization rule omits context-dependent constraint activation from the decision procedure.
\end{proof}

This proposition does not imply that scalarization is impossible in principle. Rather, it shows that scalarization is adequate only after the variables governing context-dependent constraint activation have been represented. The hard part is therefore not choosing weights over a fixed vector, but deciding which contextual variables must enter the objective state and which constraints they activate.

Scalar optimization can select clarification, refusal, uncertainty disclosure, confirmation, escalation, or abstention only if the relevant context, constraints, uncertainty, and action classes have already been represented and validated. The missing step is therefore not merely reward maximization, but objective-structure identification. The central problem for frontier AI systems is to design choice rules that can determine when scalar optimization is appropriate, when constraints must override preferences, and when the correct system behavior belongs to a different action class altogether.

\begin{table}[t]
\centering
\small
\begin{tabularx}{\linewidth}{p{0.21\linewidth}p{0.27\linewidth}X}
\toprule
Layer & Function & Example Failure If Absent \\
\midrule
Objective-State Representation
& Preserve distinct objective-relevant signals before aggregation
& A helpful answer hides factuality, privacy, uncertainty, or stakeholder risk. \\
\addlinespace
Preference-Aware Objective Decomposition
& Preserve reasons behind preference judgments
& The model learns average user satisfaction without learning when preference should be overridden. \\
\addlinespace
Context-To-Objective Routing
& Activate the right objectives, constraints, stakeholders, and escalation policy
& A high-risk or privacy-sensitive request is treated as casual conversation. \\
\addlinespace
Hierarchical and Lexicographic Constraints
& Separate non-tradeable constraints from soft preferences
& Privacy, safety, legality, or non-deception failures are compensated by fluency or usefulness. \\
\addlinespace
Deliberative Policy Reasoning
& Interpret explicit principles and resolve objective conflicts
& The system complies, refuses, or escalates without a contestable reason. \\
\addlinespace
Agentic and Tool-Use Control
& Control when model outputs become external actions
& The system efficiently executes an irreversible, unauthorized, or stakeholder-harming action. \\
\addlinespace
Controlled Personalization
& Adapt to users under consent, transparency, deletion, and non-manipulation constraints
& Personalization becomes covert profiling or hidden engagement optimization. \\
\addlinespace
Diagnostic Evaluation and Revision
& Test objective selection and update policies, thresholds, metrics, and behavior after failures
& A benchmark-passing system becomes stale, brittle, or misaligned under deployment shift. \\
\bottomrule
\end{tabularx}
\caption{A pathway for implementing contextual multi-objective optimization. Existing technical routes provide useful components, while the proposed pathway specifies how these components can be composed into a context-dependent and revisable decision process.}
\label{tab:implementation_stack}
\end{table}

\section{A Pathway for Contextual Multi-Objective Optimization}
\label{sec:implementation}

The preceding sections formulate open-ended frontier AI behavior as a context-dependent choice problem over candidate actions, objective estimates, active constraints, stakeholders, uncertainty, and conflict-resolution procedures. This formulation suggests that the implementation challenge is not simply to learn a stronger scalar reward model, but to build a decision process that can preserve objective-relevant evidence, activate the right objectives in context, distinguish constraints from preferences, select appropriate action classes, and revise these procedures after failure.

Existing technical routes provide important components for this goal. Multi-objective optimization and multi-objective reinforcement learning offer tools for vector rewards, Pareto analysis, scalarization, and preference-conditioned policies \citep{miettinen1999nonlinear,roijers2013survey,hayes2022practical,qiu2024traversing}. Preference learning, RLHF, and DPO provide scalable ways to learn from human judgments \citep{ng2000algorithms,christiano2017deep,ouyang2022training,rafailov2023direct,azar2024general}. Constitutional, deliberative, and oversight approaches make principles and evaluation procedures more explicit \citep{bai2022constitutional,guan2024deliberative,irving2018ai,burns2023weak}. Evaluation, interpretability, red teaming, and documentation provide tools for testing, diagnosis, and accountability \citep{liang2022holistic,lin2021truthfulqa,perez2022red,ganguli2022red,olah2020circuits,elhage2022toy,mitchell2019model,gebru2021datasheets,raji2020closing}.

However, these methods are best understood as components rather than complete solutions. They help estimate objectives, learn preferences, apply principles, or evaluate behavior, but they do not by themselves determine which objectives should govern a particular context, which objectives are non-tradeable constraints, whose interests should count, and when the appropriate behavior is not to answer directly but to ask for clarification, refuse, disclose uncertainty, request confirmation, or escalate. We therefore describe a pathway that connects existing techniques into a context-dependent and revisable decision process.

\subsection{Objective-State Representation}

The first layer is objective-state representation. Multi-objective optimization and multi-objective reinforcement learning motivate the idea that AI behavior should not be represented only by a single scalar objective. Classical multi-objective optimization studies vector-valued objectives, Pareto optimality, scalarization, and trade-off surfaces \citep{miettinen1999nonlinear}. Multi-objective reinforcement learning extends these ideas to sequential decision-making, where agents optimize vector rewards or learn policies conditioned on preferences over objectives \citep{roijers2013survey,hayes2022practical,qiu2024traversing}. These tools are useful because they make objective multiplicity explicit. However, for open-ended frontier AI systems, the main challenge is not only to optimize over a given vector reward, but to preserve objective-relevant evidence before deciding which objectives are active, which are tradeable, and which should operate as constraints.

Instead of training only a single reward model, a system can maintain an objective-state representation for candidate actions:
\begin{equation}
    \mathbf{L}_{\mathcal{O}_{\mathrm{active}}(c)}(a,c)=\big(L_i(a,c)\big)_{i\in \mathcal{O}_{\mathrm{active}}(c)} .
\end{equation}
Here, $\mathcal{O}_{\mathrm{active}}(c)$ denotes the context-dependent set of objective dimensions relevant to interaction context $c$, and $L_i(a,c)$ estimates how candidate action $a$ relates to objective dimension $i$ in that context. The set $\mathcal{O}_{\mathrm{active}}(c)$ is not assumed to be fixed across all tasks or complete in advance. A casual creative-writing request, a high-stakes advice request, a privacy-sensitive query, and a tool-use command may activate different objective dimensions. The role of $\mathbf{L}_{\mathcal{O}_{\mathrm{active}}(c)}(a,c)$ is therefore not to provide a universal taxonomy of AI objectives, but to preserve structured evidence before aggregation, constraint filtering, or conflict resolution.

This distinction matters because premature scalarization can hide the reason structure of a judgment. A response may be preferred because it is more useful, more truthful, more cautious, more concise, less risky, or more respectful of stakeholder interests. If these reasons are collapsed into one reward score, the system may learn that an action is better without learning why it is better or when that judgment should no longer apply. Evaluation work such as HELM and TruthfulQA supports this concern by showing that model quality cannot be fully captured by a single aggregate score: truthfulness, calibration, robustness, toxicity, fairness, and efficiency may vary independently \citep{liang2022holistic,lin2021truthfulqa}. Objective-state representation preserves such distinctions so that they can be inspected and used by later stages of the choice rule.

The representation should also leave room for uncertainty about the objective structure itself. In many frontier AI interactions, the system may not know whether a request is low-stakes or high-stakes, whether a third party is affected, whether the user is asking for general information or actionable advice, or whether a constraint should be activated. Such uncertainty should not be forced into the same score as ordinary preference satisfaction. Instead, uncertainty about context and objective relevance should be represented explicitly and passed to later layers of the choice rule, where it may trigger clarification, uncertainty disclosure, confirmation, refusal, or escalation.

Thus, objective-state representation differs from ordinary vector reward modeling in its role. It is not merely a richer reward signal for optimization. It is an intermediate representation that keeps objective-relevant evidence available for contextual routing, hierarchical constraint enforcement, and conflict-resolution procedures. Scalar optimization may still be appropriate in simple or well-specified settings, but it should operate only after the system has preserved the objective structure needed to determine whether scalarization is legitimate in the current context.

\subsection{Preference-Aware Objective Decomposition}

Preference learning is central to modern alignment because many desired behaviors cannot be fully specified by hand. Inverse reinforcement learning and preference learning formalize the problem of inferring objectives from observed or elicited human judgments \citep{ng2000algorithms,christiano2017deep}. RLHF made this approach central to language-model alignment by combining supervised fine-tuning, reward modeling, and policy optimization \citep{ouyang2022training}. DPO and related methods further simplify the pipeline by optimizing directly against preference comparisons \citep{rafailov2023direct,azar2024general}.

The difficulty is that ordinary preference data often compress heterogeneous reasons into a single comparison. A response may be preferred because it is more truthful, more cautious, more concise, more polite, less harmful, or more aligned with annotator expectations. If these reasons are not preserved, the model may learn average preference satisfaction while failing under objective conflict. In particular, user or annotator preference should not always dominate truthfulness, safety, privacy, or third-party protection.

A suitable implementation should therefore enrich preference learning with reason structure. Annotation pipelines can record not only which answer is preferred, but also why it is preferred, whether the preference depends on context, whether any constraint is violated, whether uncertainty is adequately disclosed, and whether stakeholders other than the immediate user are affected. Such labels do not replace RLHF or DPO; they make these methods more suitable for contextual multi-objective optimization by preserving information that a single preference bit cannot represent.

This decomposition also helps distinguish preference disagreement from annotation noise. Disagreement may arise because annotators emphasize different objectives, possess different expertise, or apply different domain norms. Treating all disagreement as noise encourages averaging, whereas contextual multi-objective optimization requires identifying the source of disagreement and determining which reasons should matter in the present context.

\subsection{Context-To-Objective Routing}

A general assistant should not apply the same objective mixture to every prompt. A poem request, a debugging task, a medical-style question, a legal-style question, an emotional-support conversation, a privacy-sensitive request, and a tool-use instruction activate different objectives. Context-to-objective routing maps an interaction context to an objective profile:
\begin{equation}
    \rho(c) \rightarrow
    \{
    \mathcal{O}_{\mathrm{active}}(c),
    \mathcal{C}_{\mathrm{active}}(c),
    \mathcal{S}(c),
    \mathcal{U}_{\mathrm{req}}(c),
    \mathcal{E}(c)
    \},
\end{equation}
where $\mathcal{O}_{\mathrm{active}}(c)$ denotes active objectives, $\mathcal{C}_{\mathrm{active}}(c)$ active constraints, $\mathcal{S}(c)$ relevant stakeholders, $\mathcal{U}_{\mathrm{req}}(c)$ uncertainty and calibration requirements, and $\mathcal{E}(c)$ the applicable escalation or confirmation policy.

Existing moderation, classification, and policy-routing systems provide partial technical foundations for this layer. However, contextual objective routing is richer than topic classification. It should identify not only the domain of the request, but also risk level, reversibility, evidence quality, user vulnerability, tool-use implications, stakeholder scope, and conflict type. The same surface-level request may be a benign information query, a high-stakes advice request, a privacy-sensitive request involving another person, or a request that could enable harmful action.

Routing errors are high leverage. If a high-risk context is routed as low-risk, constraints may fail to activate. If a low-risk context is routed as high-risk, the system may over-refuse. Routing should therefore be evaluated as a core alignment component, not as invisible preprocessing. A system should be tested not only on final answer quality, but also on whether it activated the correct objective profile before acting.

\subsection{Hierarchical and Lexicographic Constraints}

Some objectives should be modeled as constraints rather than preferences. AI safety research has long shown that optimizing a misspecified objective can produce unintended behavior, including reward hacking, unsafe exploration, distributional shift, and negative side effects \citep{amodei2016concrete}. Specification-gaming examples show that agents can satisfy a literal reward while violating the designer's intended objective \citep{krakovna2020specification}. Goodhart's law captures the broader pattern that a proxy measure may cease to be reliable once optimized directly \citep{goodhart1975problems}. Empirical work on reward-model overoptimization similarly suggests that optimizing a learned proxy too strongly can eventually degrade true human preference \citep{gao2023scaling}.

These concerns motivate a hierarchical or lexicographic structure. Helpfulness, fluency, concision, and style can often be treated as soft preferences. Privacy, legality, non-deception, severe safety risk, consent, and third-party protection may need to function as hard or quasi-hard constraints in relevant contexts. A choice rule can first filter actions through active constraints:
\begin{equation}
    \mathcal{A}_{\mathrm{admissible}}(c)
    =
    \{a \in \mathcal{A}(c):
    C_j(a,c)=1
    \text{ for all } C_j \in \mathcal{C}_{\mathrm{active}}(c)\}.
\end{equation}
The system can then optimize softer preferences within the admissible set:
\begin{equation}
    a^*(c) \in
    \arg\max_{a \in \mathcal{A}_{\mathrm{admissible}}(c)}
    U_{\mathrm{soft}}(a,c).
\end{equation}

This structure blocks illegitimate compensation. A privacy violation should not become acceptable because the response is helpful. A misleading answer should not become acceptable because it is reassuring. A tool action that is irreversible should not become acceptable merely because it is likely to complete the user's immediate task. In practice, constraints may be probabilistic and thresholded rather than perfectly binary, but the decision procedure should preserve the distinction between constraints and preferences.

This also explains why non-answer behaviors are part of the optimization problem. If the admissible set is empty, if constraint satisfaction is uncertain, or if the system lacks enough information to determine which constraints apply, the correct action may be to ask for clarification, refuse, disclose uncertainty, request confirmation, or escalate. These behaviors should be treated as objective-guided action classes rather than conversational imperfections.

\subsection{Deliberative Policy Reasoning}

Many objective conflicts cannot be resolved by scores alone because the relevant norms are open-textured. Terms such as helpful, safe, honest, fair, respectful, manipulative, and privacy-preserving do not have fixed operational boundaries across all contexts. Constitutional AI moves toward explicit principles rather than purely implicit preference aggregation \citep{bai2022constitutional}. Deliberative alignment similarly trains models to reason over explicit safety specifications before answering \citep{guan2024deliberative}. Scalable oversight studies how humans can evaluate model behavior that is difficult to check directly \citep{irving2018ai,burns2023weak}.

These approaches naturally support contextual multi-objective optimization, but they still require conflict-resolution procedures. Principles can conflict, oversight needs criteria, and explicit rules must be interpreted, ordered, revised, and audited. A deliberative system can help by retrieving relevant policies, identifying which objectives are in tension, and explaining why a particular action class was selected. For example, it can explain why the system answered directly, asked for clarification, disclosed uncertainty, refused, requested confirmation, or escalated.

However, deliberation should not be treated as proof of alignment. Model-generated reasoning may be incomplete, post hoc, or vulnerable to prompt manipulation. It should therefore be paired with external policy retrieval, behavioral evaluation, adversarial testing, and audits of whether the stated reason corresponds to actual behavior. Its value is that it creates an inspectable interface for conflict resolution, not that it guarantees correct objective selection by itself.

\subsection{Agentic and Tool-Use Control}

As frontier systems become more agentic, contextual multi-objective optimization must extend from response generation to external action. Tool-use systems can interact with files, messages, accounts, software environments, APIs, financial resources, and other people. This changes the objective problem: a system may correctly infer the user's immediate instruction while still selecting an inappropriate action because consent, reversibility, authorization, or third-party impact was not properly considered.

Agentic and tool-use control should therefore be treated as a distinct layer. The system should distinguish planning from acting. It may be appropriate to propose a plan, simulate consequences, or provide a dry-run preview even when it would be inappropriate to execute the plan directly. Low-impact and reversible actions may require minimal friction, while high-impact, irreversible, privacy-sensitive, or stakeholder-affecting actions should require stronger checks.

Useful mechanisms include least-authority tool access, reversible-first defaults, explicit confirmation gates, action logging, rate limits, and escalation for high-impact or uncertain actions. The core principle is that increased capability should be paired with stronger objective control. In text-only settings, objective-selection failures may produce bad answers. In tool-use settings, the same failures may produce external effects. Thus, agentic systems require not only better planning algorithms, but also explicit procedures for determining when planning may become action.

\subsection{Controlled Personalization}

Personalization can improve usefulness, but it also creates privacy, autonomy, and manipulation risks. A system that remembers a user's preferred tone, formatting style, language, or workflow conventions may provide better assistance. Yet the same mechanism can become problematic if it infers sensitive traits, vulnerabilities, emotional dependencies, or behavioral patterns for hidden engagement optimization.

Controlled personalization should therefore treat user-specific adaptation as a constrained objective modifier rather than as an unconditional improvement in user satisfaction. Low-stakes preferences, such as tone, response length, formatting style, or preferred terminology, may often be used safely. Sensitive attributes, vulnerability signals, health-related information, political or religious inferences, emotional-state predictions, or high-stakes behavioral predictions require stronger consent, transparency, deletion, and limitation controls.

Documentation and auditing practices such as model cards, datasheets, and process audits provide useful analogues because they make data assumptions, intended use, limitations, and accountability structures more explicit \citep{mitchell2019model,gebru2021datasheets,raji2020closing}. For personalization, the analogous requirement is that users should be able to inspect, correct, delete, or limit persistent information that shapes system behavior. A personalized system is not better merely because it better predicts what a user will accept; it is better only if personalization remains subordinate to privacy, autonomy, non-manipulation, and user control.

\subsection{Diagnostic Evaluation and Revision}

Because this paper is a position and perspective paper, our goal is not to introduce a new benchmark as a final contribution. Nevertheless, the framework should be empirically testable. A useful diagnostic evaluation should test whether a system can select the correct objective structure, not merely whether it can produce a fluent final answer.

One simple protocol is to construct context-paired cases. Each pair shares a similar surface request but differs in context such that different objectives, constraints, stakeholders, or action classes should be activated. For example, the same request to summarize a document may be benign when the user owns the document but privacy-sensitive when it contains another person's data. The same request to write an email may call for a draft in one context but require explicit confirmation before sending in another. The same request for advice may be a low-stakes information query in one context but a high-stakes medical-style or legal-style request in another. The same tool-use instruction may be reversible in one context but irreversible or stakeholder-affecting in another.

Such cases can be evaluated along several dimensions:
\begin{itemize}[leftmargin=1.5em,itemsep=0.2em]
    \item \textbf{Objective-routing accuracy:} whether the system activates the appropriate objective profile for the context.
    \item \textbf{Constraint recall:} whether active safety, privacy, truthfulness, legality, consent, or stakeholder constraints are identified.
    \item \textbf{Action-class accuracy:} whether the system selects answer, clarification, refusal, uncertainty disclosure, confirmation, escalation, tool use, or abstention as appropriate.
    \item \textbf{Constraint-violation rate:} whether the selected action violates any active hard or quasi-hard constraint.
    \item \textbf{Over-conservatism rate:} whether the system unnecessarily refuses, escalates, or withholds help in low-risk contexts.
    \item \textbf{Stakeholder identification:} whether the system recognizes affected parties beyond the immediate user.
\end{itemize}

This diagnostic perspective directly tests the central claim of the paper. A scalar preference model may assign similar scores to two superficially similar answers, while a contextual choice rule should distinguish them because different constraints or action classes are active. The aim is not to replace existing benchmarks, but to complement them with evaluations centered on objective selection, constraint activation, and action-class choice.

Auditing should feed revision. Policies, thresholds, routing taxonomies, objective estimators, preference data, evaluation suites, and deployment procedures should be versioned and updated when new failures appear. This is necessary because proxy objectives can drift under deployment pressure. A benchmark-passing system may become stale when users change behavior, new tools are added, new risks emerge, or models become more capable at exploiting proxy metrics.

The result is not a fixed reward function, but a revisable decision process. Contextual multi-objective optimization requires mechanisms for representing objective evidence, activating context-specific objectives, enforcing constraints, resolving conflicts, controlling external actions, limiting personalization risks, and revising procedures after failure. In this sense, implementable contextual multi-objective optimization is both a model-design problem and an institutional design problem.

\section{Limitations}
\label{sec:limitations}

This paper has several limitations. First, it provides a formulation and implementation pathway rather than a complete algorithm. The proposed framework clarifies why open-ended frontier AI behavior should be treated as contextual multi-objective optimization, but many components, including objective-state representation, context routing, constraint activation, conflict resolution, and post-deployment revision, require further technical specification and empirical validation.

Second, hierarchical constraints may introduce over-conservatism. If constraints are too broad, systems may refuse benign requests or limit user agency unnecessarily. If constraints are defined too narrowly, systems may fail to activate protections in high-risk contexts. This suggests that constraints should be transparent, contestable, and revisable rather than treated as fixed rules.

Third, decomposed objective models may create a false sense of precision. Some objectives can be estimated with useful empirical signals, while others are open-textured, context-dependent, or difficult to measure directly. Numerical scores should therefore be treated as evidence for decision-making, not as complete definitions of safety, truthfulness, fairness, autonomy, or stakeholder impact.

Fourth, context-to-objective routing may itself become a brittle proxy. A router that performs well on known categories may fail under novel tasks, adversarial prompts, ambiguous user intent, or new tool environments. The proposed framework therefore shifts part of the alignment problem from final response generation to objective-structure identification. This shift is necessary, but it does not eliminate the need for robust evaluation, auditing, and revision.

Finally, the paper focuses primarily on frontier AI assistants and agentic systems. Broader institutional questions, including who has authority to define objectives, how affected stakeholders can contest them, and how objective procedures should be revised across social and cultural contexts, remain underdeveloped. These questions are part of contextual objective regulation and require further work beyond the model-level framework developed here.

\section{Conclusion}
\label{sec:conclusion}

Modern AI progress has relied on powerful reductions: language to next-token prediction, games to win conditions, code to tests, mathematical reasoning to verifiable answers, and instruction following to preference optimization. These reductions remain indispensable when objectives are clear, stable, measurable, and legitimately tradeable. However, frontier AI systems increasingly operate in settings where the relevant objective is ambiguous, context-dependent, delayed, or only partially observable.
This paper argues that such settings should be understood as contextual multi-objective optimization problems. The central challenge is not only whether a system can optimize, but whether it can identify which objectives matter in the current context, which constraints should override ordinary preferences, whose interests are affected, what uncertainty remains, and whether the appropriate action is to answer, clarify, refuse, disclose uncertainty, request confirmation, escalate, use a tool, or refrain from acting.
The aim is not to replace scalar optimization, but to clarify when it is valid, when it breaks down, and what additional choice procedures are needed for objective-irreducible tasks. As frontier systems become more capable, personalized, and agentic, this shift from reward maximization to operational judgment becomes central to building systems that are not only powerful, but contextually appropriate, corrigible, and accountable.

\bibliographystyle{unsrt}
\bibliography{ref}

\end{document}